\RequirePackage{amsmath}
\documentclass[runningheads]{llncs}
\usepackage[T1]{fontenc}
\usepackage{helvet}  % DO NOT CHANGE THIS
\usepackage{courier}  % DO NOT CHANGE THIS
\usepackage[hyphens]{url}
\usepackage{graphicx}
\usepackage{algorithm}
\usepackage{algorithmic}
\usepackage{wrapfig}
\usepackage{tabularx}
\usepackage{fancybox}
\usepackage{multicol}
\usepackage{bbm}

\usepackage[colorinlistoftodos]{todonotes}
\usepackage[square,sort,comma,numbers]{natbib}
\usepackage{framed, tcolorbox}
\usepackage{hyperref}
\usepackage{bibentry}
\usepackage{booktabs}
\begin{document}

\title{Linear Programming based Approximation to Individually Fair k-Clustering with Outliers}
\titlerunning{Individually Fair k-Clustering with Outliers}
% If the paper title is too long for the running head, you can set
% an abbreviated paper title here
%
\author{Binita Maity\inst{1}$^*$ \and
Shrutimoy Das\inst{1}$^*$ \and
Anirban Dasgupta\inst{1}}
\def\thefootnote{*}\makeatletter\def\Hy@Warning#1{}\makeatother\footnotetext{These authors contributed equally to this work}
% %
% \authorrunning{F. Author et al.}
% % First names are abbreviated in the running head.
% % If there are more than two authors, 'et al.' is used.
% %
\institute{Indian Institute of Technology, Gandhinagar
% \email{lncs@springer.com}\\
% \url{http://www.springer.com/gp/computer-science/lncs} \and
% ABC Institute, Rupert-Karls-University Heidelberg, Heidelberg, Germany\\
\email{\{binitamaity,shrutimoydas,anirbandg\}@iitgn.ac.in}}
\maketitle              % typeset the header of the contribution
\begin{abstract}
Individual fairness guarantees are often desirable properties to have, but they become hard to formalize when the dataset contains outliers. 
Here, we investigate the problem of developing an individually fair $k$-means clustering algorithm for datasets that contain outliers. 
That is, given $n$ points and $k$ centers, we want that for each point which is not an outlier, there must be a center within the $\frac{n}{k}$ nearest neighbours of the given point. While a few of the recent works have looked into individually fair clustering, this is the first work that explores this problem in the presence of outliers for $k$-means clustering.

For this purpose, we define and solve a linear program (LP) that helps us identify the outliers. We exclude these outliers from the dataset and apply a rounding algorithm that computes the $k$ centers, such that the fairness constraint of the remaining points is satisfied. We also provide theoretical guarantees that our method leads to a guaranteed approximation of the fair radius as well as the clustering cost. We also demonstrate our techniques empirically on real-world datasets.

\end{abstract}
\section{Introduction}
The rapid adoption of machine learning (ML) algorithms  for making real-life decisions warrants that these algorithms must not be prejudiced against certain sections of society. Such unfairness have been observed in instances such as predicting future recidivism or  defaulters on credit card payments. Fairness constrained ML algorithms aim to mitigate such biases. Depending on the nature of the  problem at hand, there are a variety of fairness notions. 
% Clustering is a well-known problem in computer science. Among other,  \cite{1056489} $k$-means clustering is the most popular algorithm used in ML applications.

 Real-world data, however, is noisy, and such spurious data may affect the desired results. To handle such problems, several works have looked into designing algorithms  to handle such outliers and noise in the data \cite{pmlr-v124-deshpande20a},\cite{localout},\cite{im2020fastnoiseremovalkmeans}. Charikar \cite{Charikar2001AlgorithmsFF}  proposed the first work on $k$-median clustering with outliers, discarding $z$ points  identified as outliers, before solving the $k$-median problem on the inliers. However, this work does not consider the fairness metric, which makes our problem novel. Some works also use linear programs for this problem but suffer from large running times. \cite{krishnaswamy2018constantapproximationkmediankmeans},\cite{10.5555/1347082.1347173}. The authors in \cite{huang2024nearlinear} proposed a near-linear time approximation algorithm for this problem. These works focussed only on the clustering problems without any fairness constraints. With the introduction of the notion of individually fair clustering \cite{DBLP:journals/corr/abs-1908-09041}, the auhtors in \cite{Bateni2024ASA, pmlr-v151-chhaya22a} proposed scalable methods for individually fair $k$-clustering.  These papers do not consider the presence of outliers in the dataset.

% Other works that tackle the $\ell_p$ norm formulation of clustering with different fairness notions are \cite{DBLP:journals/corr/abs-1802-05733,pmlr-v97-backurs19a,bera2019fairalgorithmsclustering,Ahmadian_2019,schmidt2021faircoresetsstreamingalgorithms}.

% Machine Learning algorithms used everywhere from daily life necessaries to medical  imaging. But some algorithms provide "unfair" consequences like recidivism (cite). To deal with such algorithms Machine Learning(ML) community designed "Fair ML" algorithm to ensure fairness aspect of the output provided by these algorithm. Some of the previous work on defining different notions of fairness. On the other hand some researchers looked into make ML algorithms more fair. Although the definition of fairness may change depends on the  nature of the problem, but indeed there are a variety of statistical parameters to measure fairness. 

In this paper, we explore the problem of fair $k$-clustering in the presence of outliers in the dataset. Given a dataset $X$ of $n$ points and a distance metric $d,$ we want to detect and exclude the outliers, and then compute a set $S = \{u_1, \ldots, u_k \}$ of $k$ centers, such that the clustering cost is $cost  (\underset{v \in X}{\sum} d(v,S)^{p}),$ where $d(v, S) = \min_{u \in S} d(v, u)$, is minimized for the inlier points. Here,  $(p=2)$ for $k$-means and $(p=1)$ for the  $k$-median  problem. Additionally, we also want to guarantee an appropriate notion of fairness.

We consider the notion of individual fairness, which was proposed by Jung et.al \cite{DBLP:journals/corr/abs-1908-09041}. This notion states that for every point $v \in X,$ 
there is a distance threshold $r(v)$, the {\em fair radius} of $v,$ which is defined as the distance from $v$ to its $\frac{n}{k}$th nearest neighbour. We then impose a constraint that every point $v$ must have a center within a ball of radius $r(v).$ Hence, we are interested in computing a set of fair centers such that the clustering objective is minimized and individual fairness is guaranteed for each point that is an inlier. Excluding the outliers ensures that the clustering cost is not influenced by those points and the fairness violations are bounded.

% \subsection{Motivation}
% Examples
% \vspace{10 pts}

\begin{wrapfigure}[11]{L}[1pt]{4.5cm}
    \centering
    \includegraphics[width=1\linewidth]{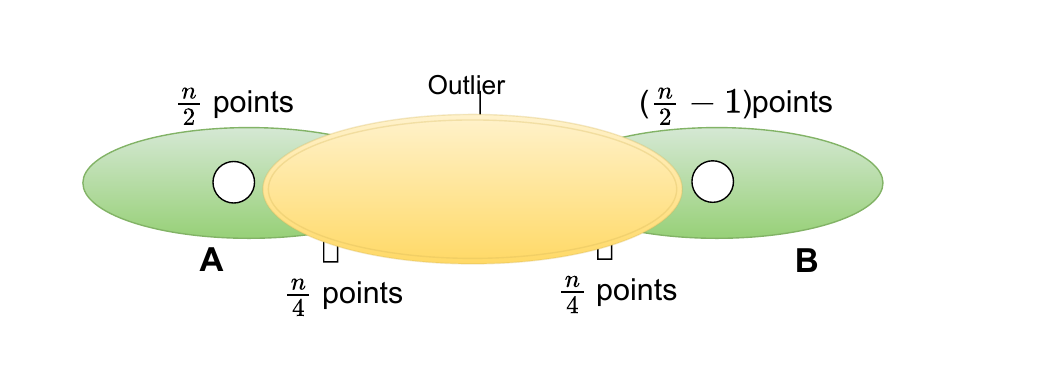}
    \caption{\label{fig:outlier_cluster}Individually fair $2$-clustering in the presence of an outlier.}
    \label{fig:outlier-example}
\end{wrapfigure}

It is easy to see that the presence of outliers can significantly alter the fair clustering objective. We give an example in Figure~\ref{fig:outlier-example}. WLOG, assume that $k = 2$ and points are in 1D. We have two (green) clusters $A$ and $B$ that are well separated, $|A| = n/2$, and $|B| = n/2 -1$. The balls indicate $n/2$ points around the corresponding centers and the outlier $p$.  There is one outlier $p$ that lies in the middle of balls $A$ and $B$. In order to satisfy the fairness condition of the outlier, the algorithm is forced to place a center in the overlapping region of the green and yellow balls. This will increase the clustering cost of either cluster $A$ or cluster $B$. 

% In Figure \ref{fig:outlier_cluster}, suppose we have two clusters, $A$ and $B,$ of $\frac{n}{2}$ and $(\frac{n}{2}-1)$ points each. There is also an outlier, as shown by the red point in the figure. Suppose we compute the cluster centers while considering the outlier. In this case, one of the cluster centers, say $A$ (represented by the dotted circle), will be close to the outlier. This is to satisfy the individual fairness condition for the outlier point.

% However, the clustering cost for this set of centers is higher than the cost when we do not consider the outlier (shown by the solid circle in cluster $A$). 
%The fair radius also changes: for a point on the left-end of cluster $A,$ the fair radius is large when clustering is done including the outliers.

As a motivating real-world example, is has been observed that during collection of data such as the weight of children aged between $0-5,$ some human errors results in absurd data, such as negative weights or some child having a weight of $100$kgs . Using such noisy data for designing government policies, such as placement of nursing centers, may result in unfair or inefficient implementations.

There have been a number of recent results for fair clustering. \cite{mahabadi2020ind} gave a local search based algorithm for both $k$-median and $k$-means clustering while \cite{negahbani2021better} improved upon the guarantees on the approximation quality of the objective and the fairness by defining a linear program (LP) for this purpose as well as proposing a rounding techniques for this LP.  A different LP formulation was also given in \cite{vakilian2022improved}. However, none of these works consider the setting where the dataset contains outliers. \cite{han2023approx} explored the individually fair $k$-center problem in the presence of outliers. The authors proposed an algorithm for minimizing the maximum fairness ratio of the inlier points. However, their work does not focus on minimizing the total cost. In this work, we explore the individually fair $k$-clustering problem in the presence of outliers. To the best of our knowledge, this is the first work to address this problem.

% We can create a critical ball using some points from each cluster that sum upto $n/k$ points, with $r(V)$ as the radius of the critical ball. 

% \begin{figure}
%     \centering
%     \includegraphics[width=1\linewidth]{pt2.drawio.pdf}
%     \label{fig:3}
% \end{figure}

% For each cluster, we have $\alpha_{i}r(v_{i})$ fair radius. If we want to maintain fairness for the outlier point, then we have to assign the outlier point to any of the existing centers; hence, we have to include the outlier with one of the existing cluster centers. The $\alpha$ will increase for that cluster; consequently, cost will also increase, 
% as $min_{u,v} \sum_{u,v \in X} d(u,v)^{p} x_{v,u}$ s.t $\forall v \in X, x_{v,u} =0 $ if $d(u,v)>r(v)$ 

% \textbf{Case 2}
% \begin{figure}
%     \centering
%     \includegraphics[width=1.25\linewidth]{shift.drawio.pdf}
%     \caption{shift delta}
%     \label{fig:enter-label}
% \end{figure}

% For the given example, if we consider the outliers while creating a cluster, then the distance of the center from the cluster of outlier point is beta. Then k means the cost will perform poorly as we are getting a very large distance.

% Now the new center will shift $\Delta$ towards the outlier. Then the fairness ratio will be $\frac{r+\Delta}{r} = 1+ \frac{\Delta}{r}$. It is possible to increase $\Delta$ by increasing R (where $R \geq n.r $). 

% [\cite{eq} gives the ratio 1 irrespective of the center (whether the center is within the cluster of points or the outlier as the center).]

% \subsection{Our Contribution}
Our main contributions can be enumerated as follows :
\begin{enumerate}
    \item We define a new LP, solving which enables us to detect the outliers, by rounding the corresponding variables.
    \item We propose OutRound, an algorithm for rounding the outlier variables and recompute the remaining variables in the objective.
    \item We show that the cost of the LP obtained after rounding the outlier variables is not much worse than the original LP cost. Hence the cost of the final approximation remains within a constant factor of the optimal. 
    \item We empirically demonstrate the utility of our proposed method on various real-world datasets.
\end{enumerate}

\subsection{Outline of the Paper}
We review some of the works that have looked into the problems of clustering, fairness and outliers in Section $1.$ We define some basic notations as well as discuss some preliminaries in Section $2.$ In Section $3,$ we give the LP formulation for our problem. We discuss our proposed algorithms in Section $4$ followed by proving the approximation guarantees in Section $5.$ We demonstrate our algorithms in the experimentation section in Section $6.$ We conclude our paper in Section $7.$

\section{Preliminaries}

% Defination 1: \textbf{$k$-means clustering} : Given a set of $n$ points we partition into $k$ sets 

We briefly describe and define the following terminologies that will be used throughout the paper.
Let $X$ denote a set of $n$ datapoints which we want to partition into $k$ clusters. In this paper, we focus on  $k$-clustering where the clustering cost can be defined as ,
\begin{align} \label{eqn:kmeans_cost}
 \underset{S \subseteq X :|S| \leq k}{min} \underset{v\in X}{\sum} d(v, S)^p   
\end{align}
where $d(v, S)$ is the distance from a point $v$ to the nearest center $u \in S.$ 
% Here we used individual fairness notion. 

% \begin{definition}Individual fairness can be defined as treating similar individual similarlly. [\cite{dwork2011fairnessawareness} ] . 
% \end{definition}

% \begin{definition}Fair radius of given set of points $v \subseteq X$ ,$ |v|=n$ can be defined as, in a metric space $(X,d)$, $k \in [v]$,  $r_{k}(v)$ to be the radius of the minimum ball centered at $v$ that contains $(n/k)$ point of $X$ [\cite{mahabadi2020ind}] .
    
% \end{definition}

% ###if need add bi-criteria approximation

% \todo{Define the fair radius function beforehand.}
\begin{definition}[\emph{Fair radius $r(\cdot)$.}] The fair radius for a point $v \in X$ is the radius of the ball containing the nearest $n/k$ neighbours  of $v.$
\end{definition}

\begin{definition}[\emph{Fair $(p,k)$-clustering}]  Given the set of datapoints $X$, the distance function $d(.,.),$  and the  fair radius function $r(v),$ the fair-$(p,k)$ clustering problem looks at minimizing the clustering cost such that the distance from a point to its center is at most $r(v).$ 
% Our goal is to find fair k center, here $s \subseteq X$.
The problem  can be defined as follows, 
\begin{align}
 \begin{split}
     \underset{S \subseteq X :|S| \leq k}{min} &\sum_{v \in X, u \in S} d(v,u)^{p} \\
    & \text{s.t } d(v, S) \leq r(v), \forall v \in X
 \end{split}   
\end{align}
\end{definition}

We will consider the setting of the problem where there are outliers present in the data. We consider the points which are far away from a given set of centers as the outliers. 

% K means objective :

% $ min_{c}\sum_{i} d(x_{i},C) $ s.t $\frac{d(x_{i},C)}{\alpha_{NR}(i)}$\\

\begin{definition}[\emph{$(\alpha, k, m)$- fair clustering excluding outliers}]
Given a set of points $X$ in a metric space $(X,d),$ a $k$-clustering using $S$ centers and $m$ outliers is $(\alpha, k)$-fair if $Z$ is denotes as outliers, $|Z| \le m$. For centers $S\subseteq X \setminus Z$, and all $v \in X \setminus Z$, $d(v, S) \leq \alpha r(v),$ where $d(v, S)$ denotes the distance of $v$ to its closest neighbour in $S$. The cost of the solution is denoted by 
\begin{align}
 \begin{split}
     \underset{Z, S \subseteq X :|S| \leq k, |Z| \le m}{min} &\sum_{v \in X\setminus Z, u \in S\setminus Z} d(v,u)^{p} \\
    & \text{s.t } d(v, S) \leq \alpha r(v), \forall v \in X\setminus Z
 \end{split}   
\end{align}

The above formulation implies that outliers are excluded from the cost calculation as well as from any fair radius guarantees. 
\end{definition}

% \begin{definition} 
%  [Triangle inequality] - Given a set of points $X$,  if $ \forall u, v, w \in X$, A distance function $d$ satisfies the $α$-approximate triangle inequality $d(u, w) \leq \alpha · (d(u, v) + d(v, w))$
%  \end{definition}

% \begin{definition} For a given set of $X$ points in the metric space $(X,d)$ , 

%     An algorithm is a (β, γ)-bicriteria approximation
% for α-fair k-clustering w.r.t. a given `p-norm cost
% function if for any set of points P in the metric
% space (X, d) the solution SOL returned by the
% algorithm on P satisfies the following properties:
% \end{definition}

\section{LP Formulation}

% \section{LP to solve fairness with outliers}
We define a linear program \eqref{lp} for solving the problem of individually fair $k$-means clustering in the presence of outliers. While \cite{negahbani2021better} proposed an LP for solving the individually fair $k$-clustering problem, it is unable to identify outliers. The LP that we propose is able to detect outliers, and is equivalent to the LP defined in \cite{negahbani2021better} when we set the number of outliers to be $0.$ In our LP formulation, we have the following variables:
\begin{itemize}
\item $\forall u,v \in X, x_{vu}$  indicates whether the  point $v$ is being assigned to the center $u$,
\item $\forall u \in X, y_u$ indicates whether a center is being opened at point $u$ ($y_u = 1$ implies $u$ is opened as a center),
\item $\forall v\in X, z_v$ indicates whether the point $v$ is being labeled as an outlier ($z_v=1$ implies $v$ is marked as an outlier).
\end{itemize}
The LP is defined as follows. 
\begin{align}
\underset{u,v}{\text{min}} & \sum_{u,v \in X} d(u,v)^{p} x_{vu} \label{lp}\tag{LP}\text{ such that }\\
& \sum_{u \in X} y_{u} \le k, \label{lp1}\tag{LP1}\\
& \sum _{v \in X} z_{v} \le m, \label{lp2}\tag{LP2}  \\
& \sum_u x_{vu} \ge 1 - z_v, \forall v \in X, \label{lp3}\tag{LP3} \\
        & x_{vu} \le y_u, \forall u, v \in X \label{lp4}\tag{LP4} \\
& y_u \le 1 - z_u, \forall u \in X, \label{lp5}\tag{LP5} \\
&  x_{v,u} =0  \text{ if } d(u,v)> \alpha r(v), \forall u,v \in X, \label{lp6}\tag{LP6}\\
&  0 \leq x_{vu}, y_{u}, z_u \le 1, \forall u, v \in X, \label{lp7}\tag{LP7}
\end{align}
where the constraints are the following:
% \begin{enumerate}
[\eqref{lp1}] implies that the total number of centers is $k$, %--- $\sum_{u \in x} y_{u} \le k$.
[\eqref{lp2}] constraints the total number of outliers to be at most $m$, %--- $\sum _{v} z_{v} \le m$.
 [\eqref{lp3}] says that every point that is not an outlier is assigned to one center, %--- $\forall v, \sum_u x_{vu} \ge 1 - z_v$.
[\eqref{lp4}] says that point $v$ is assigned to point $u$ only if $u$ is being marked as a center, %--- $\forall u, v, x_{vu} \le y_u$.
[\eqref{lp5}] implies that point $u$ is a center only if it is not an outlier, and % $\forall u, y_u \le 1 - z_u$.
[\eqref{lp6}] implies that every point is assigned to a center that is within an $\alpha$ factor of the fair radius $r(v)$. We specify this parameter in the later sections.%--- $\forall v \in X, x_{v,u} =0$  if $d(u,v)>r(v)$.
% \item  $\forall u, v, 0 \leq x_{vu}, y_{u}, z_u \le 1$
% \end{enumerate}

\section{Proposed Algorithm}
% We use filter algortihm \cite{best_possiblekcenter},\cite{pcenter}.
% % \subsubsection{Algorithms.}
% ~\ref{alg:algorithm}

We discuss our proposed algorithm in this section. The basic idea behind Algorithm \ref{alg:algorithm}, \emph{IFXO},  is as follows: We solve the linear program defined in \eqref{lp}, detect the outliers from the solution to \eqref{lp} using our rounding algorithm, \emph{OutRound} (Algorithm \ref{alg:outround}), and then apply a rounding algorithm on the remaining inlier points in order to compute the fair centers.  For rounding the latter part of the algorithm, we employ FairRound algorithm, given in \cite{negahbani2021better}.

\begin{algorithm}[ht]
\caption{\emph{Individual Fairness Excluding Outliers (IFXO)}}
\label{alg:algorithm}
\textbf{Input}: Datapoints $X$, number of centers: $k$, number of outliers : $m$\\
% \textbf{Parameter}: $p = 1$(k-median) or $2$(k-means)\\
\textbf{Output}: Set of fair centers, $S$
\begin{algorithmic}[1] %[1] enables line numbers
\STATE Solve the linear program given in (LP) to get $(x^{*}, y^{*},z^{*})$ 
\STATE $(x',y'), \texttt{outlier\_indices} \gets \emph{OutRound()}$
\STATE The set of inliers $X^{in} \gets X \setminus X[\texttt{outlier\_indices}]$
\STATE $S \gets \text{FairRound}(X^{in}, x', y').$
\STATE \textbf{return} $S$
\end{algorithmic}
\end{algorithm}

\begin{algorithm}[ht]
\caption{\emph{OutRound}}
\label{alg:outround}
\textbf{Input}: Datapoints $X$, Number of centers: $k$, Distance metric : $d,$ Number of outliers : $m,$ threshold : $\tau$, solution of \ref{lp} : $LP_{\alpha=1}(x^*, y^*, z^*)$\\
% \textbf{Parameter}: $p = 1$(k-median) or $2$(k-means)\\
\textbf{Output}: $(x', y')$ --- Recomputed variables for the inliers, \texttt{outlier\_indices}
\begin{algorithmic}[1] %[1] enables line numbers
\FOR{ each $u \in X$}
\STATE $y'_{u} \gets y_{u}$
\FOR{ each $v \in X$}
\STATE $x'_{vu} \gets x_{vu}$
\ENDFOR
\ENDFOR
\STATE Let $z' = \{\mathbf{1}[z^{*}[v] > \tau]] : \forall v \in X \},$ for some threshold, $\tau.$ Store the indices of the nonzeros in  \texttt{outlier\_indices}
\STATE for each $v \in X$ such that $z'[v] = 1,$ set $x'_{vu} = 0, \forall u \in X$
\STATE Let $U_{OUT} = \{u : y_{u} \neq 0 \wedge z'[u] = 1 \},$ the set of outliers that are also open as centers
\FOR {each $u_{OUT} \in U_{OUT}$}
\STATE \label{line:y_prime_out} Set $y'_{u_{OUT}} = 0$
% \STATE $N(u_{OUT}) = |\{v : x_{vu_{OUT}} \neq 0\}|$
\STATE \label{line:assign_u'}$u' \gets   \underset{u \in X \setminus X[\texttt{outlier\_indices}]}{\text{argmin 
 }}d(u_{OUT}, u)$  
% and $u' \in \texttt{FairBall($u_{OUT}$)}$
\STATE  $y'_{u'} \gets \text{min}(y_{u'} +  y_{u_{OUT}}, 1)$%/ N(u_{OUT})$ 
 // set the point nearest to $u_{OUT}$ as the center
\FOR{$\forall v \in X$ such that $x_{vu_{OUT}} \neq 0$}
% \IF{$d(v,u') \leq r(v)$}
\STATE \label{line:x_prime_vu} $x'_{vu'} \gets x_{vu'} + x_{vu_{OUT}}$
\STATE // Clip $x'_{vu'}$ such that $x'_{vu'} \leq y_{u'}$
% \ELSE
% \STATE  $y'_{v} \gets \text{min}(y_v +  y_{u_{OUT}}, 1)$%/ N(u_{OUT})$ 
% // Set $v$ itself as the center for $v$%or $y'_{v} \gets y_v$
% \STATE $x'_{vv} \gets x_{vv} + x_{vu_{OUT}}$
% \STATE // Clip $x'_{vv}$ such that $x'_{vv} \leq y'_{v}$
% \ENDIF
\ENDFOR
\ENDFOR
\STATE \textbf{return} $(x',y'),$ \texttt{outlier\_indices}
\end{algorithmic}
\end{algorithm}

Our algorithm aims to create $k$ clusters of the dataset $X,$ such that the individual fairness constraint is satisfied for the inlier points. Hence, the first step is to detect and exclude the outliers. We solve \eqref{lp} for this purpose. We round the variables associated with the outliers by thresholding the values of $z^*.$ This gives us the indices of the points marked as outliers. However, the variables of \eqref{lp} also depend on the values of $z^{*},$ as shown by the constraints in \eqref{lp3} and \eqref{lp5}. Thus, rounding $z^{*}$ will also affect these variables. Specifically, suppose a point is marked as an outlier. Then, it will not be assigned to an open center. For such a point $u_{OUT},$ we can safely set $x'_{u_{OUT}u} = 0$ for all open centers $u.$ However, if $u_{OUT}$ is also open as a center, then we cannot set $x'_{vu_{OUT}} = 0$ directly. This is because  for all points $v$ that have been assigned to $u_{OUT},$ they have to be reassigned to a new center such that the constraint \eqref{lp3} remains satisfied. These cases have been handled in Algorithm \ref{alg:outround}. We note here that Algorithm \ref{alg:outround} outputs recomputed values of the $x$ and $y$ variables after thresholding the $z$ values. These recomputed $x$ and $y$ values are clipped such that the constraints of \eqref{lp} are satisfied. Also, the solution to \eqref{lp} satisfies \eqref{lp6} with $\alpha=1$ whereas after the \emph{OutRound} procedure, the constraint in \eqref{lp6} is satisfied with $\alpha=2.$
%In Figure \ref{fig:mark_outliers}, we mark the outliers that are detected after rounding the $z^{*}$ variables.

% \begin{figure}[h]
%     \centering
%     \includegraphics[width=1\linewidth]{select_scattered_bank_3d_outliers_adult_5out_baseline.pdf}
%     \caption{\label{fig:mark_outliers}$3D$ scatter plot of the Adult dataset for $600$ points, with the outliers and fair centers highlighted. Note that the range of coordinates along the `x2' axis is much larger than the other two.}   
% \end{figure}

\section{Approximation Guarantees}

\begin{theorem}\label{thm1}
    Suppose the optimal cost for \eqref{lp} is  $LP_{\alpha=1}(x^*, y^*, z^*).$ Running Algorithm \ref{alg:outround} with $(x^*, y^*, z^*)$ as inputs  and threshold $\tau = 0,$ results in the cost $LP_{\alpha=2}(x',y',z'),$ where $z' = \{\mathbf{1}[z^{*}[v] > \tau] : \forall v \in X \}.$ Then, 
    \begin{align}\label{eqn:thm1}
    LP_{\alpha=2}(x',y',z') \leq 3 LP_{\alpha=1}(x^*, y^*,z^*).
    \end{align}
    Consequently, since the FairRound algorithm in \cite{negahbani2021better} gives a $4$ approximate solution for individually fair $k$-means clustering and $8$ approximate solution for $k$-median clustering, the algorithm \texttt{IFXO} given in \ref{alg:algorithm} gives a $12$ approximate solution to the optimal for $k$-means and $24$ approximate solution to the optimal for $k$-median clustering. 
\end{theorem}

We will prove Equation \eqref{eqn:thm1} here. The latter part of the theorem can be proved using the proof techniques in \cite{negahbani2021better}. Note that we do not show any bounds for the number of outliers detected. Empirically, we observed (Table \ref{tab:k_m}) that our LP marks at most $3m$ points as outliers when we set $\tau = 0.$ Varying the value of $\tau$ controls the number of outliers detected. For higher threshold values $\tau,$ the number of outliers detected will be lower. The proof of Theorem \ref{thm1} holds when $\tau =0.$ This ensures that the $x_{vu}$ values always decrease whenever a point $v$ is marked as an outlier. 
. %for $m=10$.\todo{Can we remove the $m=10$ qualifier here?} 
In the discussions that follow, for a point $v \in X,$ we define $\texttt{FairBall(v)}$ as $\{u: u \in X \wedge d(v,u) \leq r(v) \},$ where $r(v)$ is the fair radius for $v.$ We use the following lemma (Lemma \ref{lem:duv}) for our proof  of Theorem \ref{thm1}.

% \begin{figure}[h]
%     \centering
%     \includegraphics[width=1.2\linewidth]{outround.drawio (1).pdf}
%     \caption{Intuition behind Lemma \ref{lem:duv}.}
%     \label{fig:outround}
% \end{figure}

\begin{lemma}\label{lem:duv}
  Let $u_{OUT} \in X$ be a point such that it is open as a center and it has been marked as an outlier. We compute $u' \in X$ as given in line \ref{line:assign_u'} of Algorithm \ref{alg:outround}. Let $v \in X$ be another point such that $x_{vu_{OUT}} \neq 0.$ Then,
    \begin{align}\label{eqn:lem1}
        d(v,u') \leq 2d(v, u_{OUT}) \leq 2r(v).
    \end{align}
\end{lemma}

\begin{proof}
Since $x_{vu_{OUT}} \neq 0$, $d(v, u_{OUT})\le r(v)$. Note that by choice of $u'$, $d(u_{OUT}, u') \le d(u_{OUT}, v)$. 
Hence, by triangle inequality, 
$d(v, u') \le d(v, u_{OUT}) + d(u_{OUT}, u') \le 2r(v)$.
\end{proof}

% \begin{lemma}\label{lem:duv}
%     Let $u_{OUT} \in X$ be a point such that it is open as a center and it has been marked as an outlier. We compute $u' \in X$ as given in line \ref{line:assign_u'} of Algorithm \ref{alg:outround}. Let $v \in X$ be another point such that $x_{vu_{OUT}} \neq 0.$ Then,
%     \begin{align}\label{eqn:lem1}
%         d(u_{out},u') \leq d(u_{out},v).
%     \end{align}
% \end{lemma}
% \begin{proof}
%     The solution to \ref{lp} satisfies the  constraint at \ref{lp5} implying that there may be points, $u \in X,$ for which $y_u \neq 0,$ i.e., they are open as centers. Thus, when searching for $u',$ the point $v$ is also a candidate for replacing $u_{OUT}.$ But $u'$ is selected as it is closest to $u_{OUT}$ according to $d(u_{OUT}, u').$ 

%     If $u' \notin Fairball(v),$  and we are unable to find any other such $u'$, then assign $v$ to itself as the center.  Thus, $d(u_{out},u'=v) = d(u_{out},v)$ holds in that case.
% \end{proof}

\textbf{Proof of Theorem \ref{thm1} :}
\begin{proof}
The OutRound$(x^*, y^*, z^*)$ algorithm outputs $(x',y'),$ which are the recomputed values for $(x,y)$ variables after rounding the set of variables $z^*.$  Let $z'$ be the set of variables that we get after rounding  $z^*.$ 
% Round with respect to $z$ to get $(x^{'},y^{'})$. We want to bound $LP(x^{'},y^{'}) \leq C. LP(x^{*},y^{*},z^{*}) $ Rounding scheme : Threshold $z$ 
Suppose a point is marked as an outlier, i.e,  for some $u_{OUT} \in X, \text{ }z'[u_{OUT}] = 1 .$ If $y_{u_{OUT}} = 0,$ then we need not make any changes. Otherwise, if it is also open as a  center, $y_{u_{OUT}} \neq 0,$  there will be some  points $v$ such that $u_{OUT} \in \texttt{Fairball($v$)}$ and thus $x_{vu_{OUT}} \neq 0.$

In line \ref{line:y_prime_out} of OutRound, we set $y'_{u_{OUT}} = 0.$ However, we have to reassign the amount $x_{vu_{OUT}}$ such that the constraint in \ref{lp3} is satisfied.
% Suppose, $u'$ is the closest point to $u_{out},$ as computed in  line \ref{line:assign_u'} of Algorithm \ref{alg:outround}. %such that $u'$ is also open as a center. 
% From Lemma \ref{lem:duv}, we have
% \begin{align}
% d(u_{out},u') \leq d(u_{out},v)    
% \end{align}

In line \ref{line:x_prime_vu} of Algorithm \ref{alg:outround}, for all the points that have been assigned to a center that is in $U_{OUT},$ we update $x'_{vu'} \gets x_{vu'} + x_{vu_{OUT}}.$ Then, using Lemma \ref{lem:duv}, for a point $v \in X,$ we have,

\begin{align}
\begin{split}
    d(v,u')x'_{vu'}&=d(v, u') (x_{vu'} + x_{vu_{OUT}}) \\
    &\leq d(v,u')x_{vu'} + 2d(v,u_{out})x_{vu_{OUT}}
\end{split}
\end{align}

Let $X_{OUT}$ be the set of points that have been marked as outliers.
The contribution of the remaining inliers to the LP cost is

% \begin{align}\label{eqn:sum_ineq}
% \begin{split}
%     \underset{\substack{{v\in X,} \\ {u' \in U'}}}{\sum}d(v,u')x'_{vu'} \leq \underbrace{\underset{\substack{{v\in X,} \\ {u' \in U'}}}{\sum} d(v,u')x_{vu'}}_{A} +  
%         &\underbrace{\underset{\substack{{v\in X,} \\ {u_{OUT} \in U_{OUT}}}}{\sum} 2d(v,u_{OUT})x_{vu_{OUT}}}_{B}
% \end{split}
% \end{align}

\begin{align}\label{eqn:sum_ineq}
\begin{split}
    \underset{v, u' \in X \setminus X_{OUT}}{\sum}d(v,u')x'_{vu'} &\leq \underbrace{\underset{v \in X \setminus X_{OUT}}{\sum} \underset{u' \in X \setminus X_{OUT}}{\sum} d(v,u')x_{vu'}}_{A} +  \\
        &\underbrace{\underset{v \in X \setminus X_{OUT}}{\sum}\underset{u_{OUT} \in U_{OUT}}{\sum} 2d(v,u_{OUT})x_{vu_{OUT}}}_{B}
\end{split}
\end{align}
 
% In line \ref{line:x_prime_vu} of Algorithm \ref{alg:outround}, we update $x'_{vu'} \gets x_{vu'} + x_{vu_{OUT}}.$ Then, using triangular inequality and Lemma \ref{lem:duv} , we have,
% \begin{align}
% \begin{split}
%     d(v,u') \leq & d(v,u_{OUT}) + d(u_{OUT},u') \\
%     \implies  d(v,u') \leq & 2d(v,u_{OUT})
% \end{split}
% \end{align}

% For a point $v \in X,$ 
% \begin{align}
%     d(v,u')x'_{vu'} \leq d(v,u')x_{vu'} + 2d(v,u_{out})x_{vu_{OUT}}
% \end{align}

% Thus, 
% \begin{align}\label{eqn:sum_ineq}
% \begin{split}
%     \underset{v,u' \in X}{\sum}d(v,u')x'_{vu'} \leq \underbrace{\underset{v,u' \in X}{\sum} d(v,u')x_{vu'}}_{A} +  
%         &\underbrace{\underset{v,u' \in X}{\sum} 2d(v,u_{OUT})x_{vu_{OUT}}}_{B}
% \end{split}
% \end{align}

% For part \textit{A} in Equation \eqref{eqn:sum_ineq}, we have
% \begin{align}\label{eqn:part_a_bound}
%     \underset{v, u' \in X}{\sum}d(v,u')x_{vu'} \leq \underset{v,u \in X}{\sum} d(v,u)x_{vu}, 
% \end{align}

For part \textit{A} in Equation \eqref{eqn:sum_ineq}, we have
\begin{align}\label{eqn:part_a_bound}
    \underset{v \in X \setminus X_{OUT}}{\sum} \underset{u' \in X \setminus X_{OUT}}{\sum}d(v,u')x_{vu'} \leq \underset{v,u \in X}{\sum} d(v,u)x_{vu}, 
\end{align}
 since the RHS is over a larger set of points. %and the $x_{vu}$ values never increase for the remaining points, when $\tau = 0.$ %For  part \textit{B}, for each $u_{out},$ 
% %it will be mapped to a unique $u^{'}$. Let $U^{'}$ = { $u_{out}$: $u^{'}$ is nearest to $u_{out}$}
% note that while $x_{vu_{OUT}} \neq 0,$ the quantity $d(u_{OUT},v)x_{u_{OUT},v} = 0,$ since we do not assign center to an outlier. Thus, the quantity $d(v,u_{OUT})x_{vu_{OUT}}$ appears only once in the summation of part \textit{B} of Equation \eqref{eqn:sum_ineq}. Thus, we have
% \begin{align}\label{eqn:part_b_bound}
%     \underset{v,u' \in X}{\sum} 2d(v,u_{OUT})x_{vu_{OUT}} \leq 2\underset{v,u \in X}{\sum} d(v,u)x_{vu}
% \end{align}

For part \textit{B}, the following inequality holds. The LHS summation is again over a subset of points. Hence,
\begin{align}\label{eqn:part_b_bound}
    \underset{\substack{{v\in X \setminus X_{OUT},} \\ {u_{OUT} \in U_{OUT}}}}{\sum} 2d(v,u_{OUT})x_{vu_{OUT}} \leq 2\underset{v,u \in X}{\sum} d(v,u)x_{vu}.
\end{align}

% Then, for any $u_{1},u_{2} \in X$, $U^{'}_{1} \cap U^{'}_{2} = \phi $

% Therefore, for \textit{B}, each $u_{out}$ will appear in the sum exactly once. So, we can say that,
% \begin{align}
%     \sum_{u^{'}} \sum_{v} d(v,u_{OUT})x_{vu_{OUT}} \leq \underset{v,u \in X}{\sum} d(v,u)x_{vu}
% \end{align} Here, $u_{out}$ has been replaced by $u^{'}$.

Therefore, from Equations \eqref{eqn:sum_ineq}, \eqref{eqn:part_a_bound}, \eqref{eqn:part_b_bound}, we have

\begin{align}
\begin{split}
\underset{v,u' \in X \setminus X_{OUT}}{\sum}d(v,u')x'_{vu'} \leq &\underset{v,u \in X}{\sum} d(v,u)x_{vu} +2 \underset{v,u \in X}{\sum} d(v,u)x_{vu}\\
\implies LP_{\alpha=2}(x',y') \leq & 3LP_{\alpha=1}(x^{*},y^{*},z^{*})
\end{split}
\end{align} 

Let $FR(x',y')$ refer to the cost that we get after running the FairRound algorithm. Thus, from Theorem $1$ of \cite{negahbani2021better}, we get,
\begin{itemize}
    \item for $k$-means clustering
    \begin{align}
    FR(x',y') \leq 4LP_{\alpha=2}(x',y') \leq 12LP_{\alpha=1}(x^{*},y^{*},z^{*}) \leq 12OPT_{2}, 
\end{align}
\item for $k$-median clustering
\begin{align}
    FR(x',y') \leq 8LP_{\alpha=2}(x',y') \leq 24LP_{\alpha=1}(x^{*},y^{*},z^{*}) \leq 24OPT_{1}. 
\end{align}
\end{itemize}
    
where $OPT_{2}$ is the optimal solution for an integer programming formulation of \eqref{lp} for $k$-means and $OPT_{1}$ is the corresponding optimal for $k$-median clustering.
\end{proof}

\subsection{Comment on the approximation for fair radius}
In \cite{negahbani2021better} it is shown that for each $v \in X,$ the FairRound algorithm gives an $8$ approximation to the fair radius. In Lemma \ref{lem:duv}, we show that the fair radius is violated by a factor of atmost $2,$ when the center is reassigned. Thus, we will  have a $16$-approximation to the fair radius for the inliers, both for $k$-means as well as $k$-median clustering. 

% \begin{figure}[h]
\begin{wrapfigure}[14]{L}[1pt]{5cm}
    \centering
    \includegraphics[width=\linewidth]{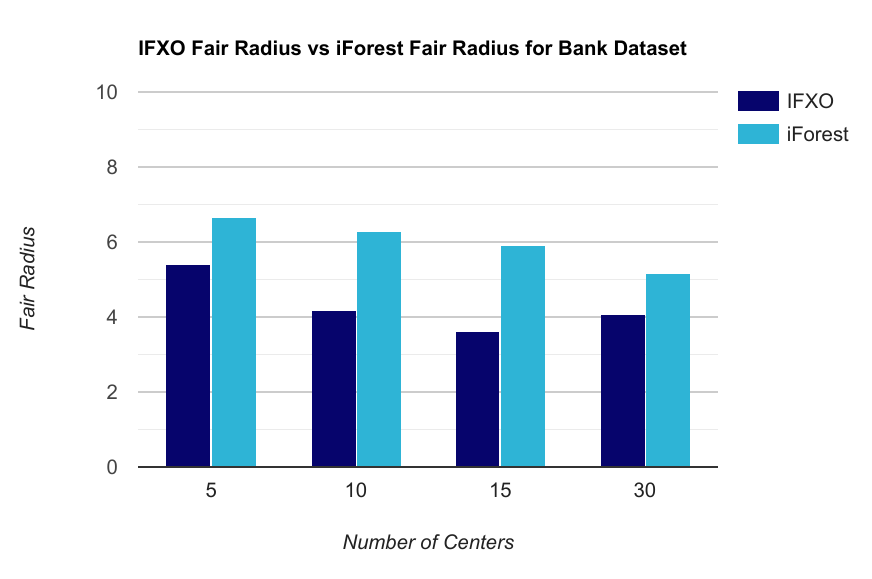}
    \caption{\label{fig:fair_radii_kmeans}Maximum Fairness Radius for the \emph{Bank} dataset for different number of clusters.}
% \end{figure}
\end{wrapfigure}

% \section{Bound for m}
% \begin{align}
%     P_{i}[o_{i}] = z_{i}
%     \sum z_{i} \leq m
%     \sum o_{i} = \text{Number of outliers}
%     \math{E}[\sum o_{i}]= \sum z_{i} \leq m
%     \sigma_{i}^2 = z_{i}(1-z_{i})
%     \sum \sigma_{i}^2 \leq \sum z_{i} \leq m 
%     P[\sum o_{i} \geq t] = (e^{- \frac{(t-m)^2}{\sum \sigma_{i}^2} }) =  (e^{- \frac{(t-m)^2}{m} }) \leq \delta  
%     \frac{(t-m)^2}{m}  \leq \sqrt{log \frac{1}{\delta}}
%     t \leq (t-m)^2 \leq m* \sqrt{log \frac{1}{\delta}}
%     t \leq m+ \sqrt{m* \sqrt{log \frac{1}{\delta}}}
% \end{align}

\section{Experiments}
We tested our proposed method on a variety of datasets and report the results. We experimented on the following datasets from UCI machine Learning Repository \cite{UCI}. These datasets are standard benchmarks for fair $k$-clustering \cite{DBLP:journals/corr/abs-1802-05733,chen2020proportionallyfairclustering,backurs2019scalablefairclustering,bera2019fairalgorithmsclustering,huang2019coresetsclusteringfairnessconstraints}. 
\begin{itemize}
    \item \textit{Bank:} This is Portuguese Bank dataset\cite{misc_bank_marketing_222}. There are three features age, balance and duration-of-account with 45211 points. 
    \item \textit{Adult:} This dataset corresponds to the 1998 US Census  \cite{misc_adult_2} with five features age, fnlwgt, education-num, capital-gain and hours-per-week with 48842 points.
    \item \textit{Diabetes:}  This dataset contains  two features age and time-in-hospital, with 101,766 points. %This dataset contains information about diabetes patient from 1999 to 2008 , at 130 hospitals, US\cite{misc_diabetes_130-us_hospitals_for_years_1999-2008_296}.
\end{itemize}
We note here that solving the LP \eqref{lp} presents a computational bottleneck for our method. Thus, as done is \cite{negahbani2021better}, we use 1000 randomly sampled points for our experiments. In order to solve \eqref{lp}, we used  the  \texttt{CPLEX}\cite{cplex2009v12} optimizer (academic version). We conduct our experiments on a system with Intel(R) Xeon(R) Gold 5120 CPU and Ubuntu version 22.04 LTS.

% \begin{table}[h]
% \begin{minipage}[l][0.6\textwidth]
\begin{wraptable}[15]{R}[3pt]{5.5cm}
\centering
\begin{tabular}{|llll|}
\hline
\multicolumn{4}{|l|}{Dataset Name: Bank}                                      \\ \hline
% \multicolumn{1}{|l|}{Centers} & \multicolumn{1}{l|}{Fm=0} & \multicolumn{1}{l|}{m=10} & \multicolumn{1}{l|}{m=10 \newline (iForest)} \\ \hline
\multicolumn{1}{|l|}{\textbf{k}} & \multicolumn{1}{l|}{$m=0$} & \multicolumn{1}{l|}{$m=10$} & \multicolumn{1}{l|}{\begin{tabular}[c]{@{}l@{}}$m=10$\\ (iForest)\end{tabular}} \\ \hline
\multicolumn{1}{|l|}{5}       & \multicolumn{1}{l|}{34.77}  & \multicolumn{1}{l|}{\textbf{27.55}}  & \multicolumn{1}{l|}{29.99} \\ \hline
\multicolumn{1}{|l|}{10}      & \multicolumn{1}{l|}{30.12}  & \multicolumn{1}{l|}{\textbf{22.04}} & \multicolumn{1}{l|}{23.99} \\ \hline
\multicolumn{1}{|l|}{15}      & \multicolumn{1}{l|}{25.71}  & \multicolumn{1}{l|}{\textbf{19.20}} & \multicolumn{1}{l|}{20.86} \\ \hline
% \multicolumn{1}{|l|}{20}      & \multicolumn{1}{l|}{16451.36146}  & \textbf{9669.918252} & \\ \hline
% \multicolumn{1}{|l|}{25}      & \multicolumn{1}{l|}{12389.47513}  & \textbf{8459.774938} & \\ \hline
\multicolumn{1}{|l|}{30}      & \multicolumn{1}{l|}{19.03}  & \multicolumn{1}{l|}{\textbf{14.86}} & \multicolumn{1}{l|}{16.01} \\ \hline
\end{tabular}
\caption{\label{tab:kmeans_wo_outliers}Comparison of the cost after running \emph{IFXO} algorithm, removing $(m=10)$ and without removing $(m=0)$ outliers for $k$-means clustering. The third column displays the cost of running the FairRound algorithm after removing $10$ outliers using iForest.}
% \end{table}
% \end{minipage}
\end{wraptable}

As a baseline, we detect and remove the outliers using the iForest \cite{liu2008iforest} algorithm and then apply the FairRound algorithm on the remaining inlier points. Also, for each of the experiments, we apply standard scaling to the dataset in order to have $0$ mean and unit variance.

% \begin{algorithm}[h]
% \caption{\emph{Baseline}}
% \label{alg:baseline}
% \textbf{Input}: Datapoints $X$, number of centers: $k$, number of outliers : $m$\\
% % \textbf{Parameter}: $p = 1$(k-median) or $2$(k-means)\\
% \textbf{Output}: Set of fair centers, $S$
% \begin{algorithmic}[1] %[1] enables line numbers
% \STATE Compute the mean for all of the datapoints in $X$
% \STATE Based on the distances from the mean, mark the farthest $m$ points as outliers
% \STATE The set of inliers $X^{in} \gets X \setminus outliers$
% \STATE Solve the linear program given in \cite{negahbani2021better} to obtain $(x^{*}, y^{*})$
% \STATE $S \gets \text{FairRound}(X^{in}, x^*, y^*).$
% \STATE \textbf{return} $S$
% \end{algorithmic}
% \end{algorithm}

% Computation power: 256 GB RAM, 
We set $m$, the number of outliers, as $1\%$ of the dataset size and  observed that the number of nonzero values for the variable $z$ in \eqref{lp} is  at most $3m.$ %for $m=10.$ 
Thus, we set the threshold $\tau=0$ for marking the outlier points. We report the number of outliers detected in Table \ref{tab:k_m}.

\paragraph{\textbf{Introducing outliers}}
Since we are using a sample of $1000$ datapoints, it is possible that there might not be outliers present in the sample. Hence, in order to artificially introduce outliers, we randomly sampled $1\%$ of the already sampled points and added a uniform noise to each of the features of these points. For a feature $col,$ we add a noise sampled from $Uniform(0,col\_max),$ where $col\_max$ is the maximum feature value for $col$ over all the points. This method of generating outliers have been used in \cite{localout}.

\begin{table}[ht]
\centering
\begin{tabular}{@{} lcccccc @{}}
\toprule
& \multicolumn{3}{c}{p=1} & \multicolumn{3}{c}{p=2}                       \\ 
\cmidrule(r){2-4}\cmidrule(l){5-7}
\textbf{k} & \textbf{Bank} & \textbf{Adult} & \textbf{Diabetes} & \textbf{Bank} & \textbf{Adult} & \textbf{Diabetes}   \\
5     & 19      & 18   & 30     & 27  & 17      & 10        \\
10     & 18      & 28   & 27     & 23   & 21      & 22      \\
15 & 25      & 26   & 11     & 23    & 17      & 10       \\
20     & 15      & 15   & 15     & 17    & 10      & 10       \\
\bottomrule
\end{tabular}
\caption{\label{tab:k_m}Number of outliers detected after solving \eqref{lp} and thresholding with $\tau=0,$ for $k$-median ($p=1$) and $k$-means ($p=2$) clustering. For this experiment, we set $m=10.$}
\end{table}

\begin{table}[ht]
\centering
\begin{tabular}{@{} lcccccc @{}}
\toprule
& \multicolumn{3}{c}{p=1} & \multicolumn{3}{c}{p=2}                       \\ 
\cmidrule(r){2-4}\cmidrule(l){5-7}
\textbf{k} & $LP^*$ & $LP'$ & $FR$ & $LP^*$ & $LP'$ & $FR$   \\
5  \quad\quad   & 591.24   \quad   & 559.13 \quad   & 572.81 \quad  \quad   & 21.68 \quad  & 21.68 \quad      & 21.68        \\
10   \quad\quad  & 392.24  \quad    & 374.63 \quad  & 368.50 \quad \quad   & 15.54   & 15.03      & 15.06      \\
15 \quad\quad & 284.64  \quad    & 283.34 \quad  & 283.34 \quad\quad    & 11.95    & 11.95      & 11.95       \\
30 \quad\quad     & 128.49  \quad    & 128.91 \quad  & 125.32 \quad\quad    & 7.54    & 7.54      & 7.54       \\
\bottomrule
\end{tabular}
\caption{\label{tab:diabetes_lp_compare}Comparison of the $k$-median ($p=1$) and $k$-means($p=2$) costs $LP^* = LP(x^*,y^*,z^*),$  $LP'=LP(x',y')$ and $FR=FairRound(x',y')$ for the \emph{Diabetes} dataset. Note that in some of the cases the change in costs across the different stages of the LP rounding is very small (requiring more than just two decimal places of precision).}
\end{table}

In Figure \ref{fig:fair_radii_kmeans}, we compare the maximum fair radius for the points after running  \emph{IFXO} and the baseline algorithm. \emph{IFXO} always gives a smaller fair radius than the baseline.

In Table \ref{tab:kmeans_wo_outliers}, we compare the cost obtained after running \emph{IFXO} and iForest on the \emph{Bank} dataset for $k$-means clustering.  Since, we remove the outliers, the costs in the second column are lower. It can also be observed that the costs obtained using our method is lower than using iForest for removing outliers. 

In Tables  \ref{tab:diabetes_lp_compare}, and \ref{tab:bank_lp_compare},  we compare the $k$-median and $k$-means clustering costs mentioned in Theorem \ref{thm1}. Empirically, it was observed that most of the $x_{vu},$ $y_v$ and $z_v$ values are zeros. So, the OutRound algorithm just sets the  $x_{vu}$ and $y_v$ values to zero if $z'[v] > 0.$ Thus, the cost LP$(x', y')$ is computed on a smaller number of points with no increase in the $x_{vu}$ values. Hence, we observe that LP$(x', y') < $ LP$(x^*, y^*, z^*).$ As expected, the FR$(x', y')$ costs are greater than LP$(x', y').$ 
\begin{table}[ht]
\centering
\begin{tabular}{@{} lcccccc @{}}
\toprule
& \multicolumn{3}{c}{p=1} & \multicolumn{3}{c}{p=2}                       \\ 
\cmidrule(r){2-4}\cmidrule(l){5-7}
\textbf{k} & $LP^*$ & $LP'$ & $FR$ & $LP^*$ & $LP'$ & $FR$   \\
5  \quad\quad   & 732.30   \quad   & 702.34 \quad   & 697.29 \quad  \quad   & 29.30 \quad  & 27.33 \quad      & 27.55        \\
10   \quad\quad  & 574.34  \quad    & 554.08 \quad  & 555.09 \quad \quad   & 23.07   & 22.14      & 22.04      \\
15 \quad\quad & 496.24  \quad    & 470.31 \quad  & 471.70 \quad\quad    & 20.08    & 19.12      & 19.20       \\
30 \quad\quad     & 378.10  \quad    & 368.53 \quad  & 367.83 \quad\quad    & 15.31    & 14.75      & 14.86       \\
\bottomrule
\end{tabular}
\caption{\label{tab:bank_lp_compare}Comparison of the $k$-median ($p=1$) and $k$-means($p=2$) costs $LP^* = LP(x^*,y^*,z^*),$  $LP'=LP(x',y')$ and $FR=FairRound(x',y')$ for the \emph{Bank} dataset.}
\end{table}

% \begin{table}[h]
% % \begin{wraptable}[14]{R}[1pt]{4.5cm}
% \centering
% \begin{tabular}{@{} lcccccc @{}}
% \toprule
% & \multicolumn{3}{c}{p=1} & \multicolumn{3}{c}{p=2}                       \\ 
% \cmidrule(r){2-4}\cmidrule(l){5-7}
% \textbf{k} & $LP^*$ & $LP'$ & $FR$ & $LP^*$ & $LP'$ & $FR$   \\
% 5  \quad\quad   & -   \quad   & - \quad   & - \quad  \quad   & 45.35 \quad  & 44.61 \quad      & 45.23        \\
% 10   \quad\quad  & -  \quad    & - \quad  & - \quad \quad   & 37.94   & 37.03      & 37.10      \\
% 15 \quad\quad & -  \quad    & - \quad  & - \quad\quad    & 34.19    & 33.60      & 34.27       \\
% 30 \quad\quad     & -  \quad    & - \quad  & - \quad\quad    & 28.30    & 27.84      & 28.26       \\
% \bottomrule
% \end{tabular}
% \caption{\label{tab:adult_lp_compare}Comparison of the $k$-median ($p=1$) and $k$-means($p=2$) costs $LP^* = LP(x^*,y^*,z^*),$  $LP'=LP(x',y')$ and $FR=FairRound(x',y',z')$ for the \emph{Adult} dataset.}
% \end{table}
% \end{wraptable}
% \vspace{10 pts}

\section{Conclusion and Future Work}
We presented a novel LP formulation for the problem of individually fair $k$-clustering for the case of both $k$-median and $k$-means. We presented the OutRound algorithm using which we derived the  approximation guarantees to the optimal. We also demonstrated our algorithm on three popular datasets. One direction of future work would be to derive bounds on the number of outliers detected by the OutRound algorithm. %It will be interesting to see how local search based methods work compared to our LP based methods, in terms of approximation quality.

\begingroup
\renewcommand{\clearpage}{}

\bibliography{ref}

\begin{thebibliography}{10}
\providecommand{\url}[1]{\texttt{#1}}
\providecommand{\urlprefix}{URL }
\providecommand{\doi}[1]{https://doi.org/#1}

\bibitem{backurs2019scalablefairclustering}
Backurs, A., Indyk, P., Onak, K., Schieber, B., Vakilian, A., Wagner, T.: Scalable fair clustering (2019)

\bibitem{Bateni2024ASA}
Bateni, M., Cohen-Addad, V., Epasto, A., Lattanzi, S.: A scalable algorithm for individually fair k-means clustering. ArXiv  (2024)

\bibitem{misc_adult_2}
Becker, B., Kohavi, R.: {Adult} (1996)

\bibitem{bera2019fairalgorithmsclustering}
Bera, S.K., Chakrabarty, D., Flores, N.J., Negahbani, M.: Fair algorithms for clustering (2019)

\bibitem{Charikar2001AlgorithmsFF}
Charikar, M., Khuller, S., Mount, D.M., Narasimhan, G.: Algorithms for facility location problems with outliers. In: ACM-SIAM Symposium on Discrete Algorithms (2001)

\bibitem{10.5555/1347082.1347173}
Chen, K.: A constant factor approximation algorithm for k-median clustering with outliers. In: Proceedings of the Nineteenth Annual ACM-SIAM Symposium on Discrete Algorithms. SODA '08 (2008)

\bibitem{chen2020proportionallyfairclustering}
Chen, X., Fain, B., Lyu, L., Munagala, K.: Proportionally fair clustering (2020)

\bibitem{pmlr-v151-chhaya22a}
Chhaya, R., Dasgupta, A., Choudhari, J., Shit, S.: On coresets for fair regression and individually fair clustering. In: Proceedings of The 25th International Conference on Artificial Intelligence and Statistics. pp. 9603--9625 (2022)

\bibitem{DBLP:journals/corr/abs-1802-05733}
Chierichetti, F., Kumar, R., Lattanzi, S., Vassilvitskii, S.: Fair clustering through fairlets. CoRR  (2018)

\bibitem{cplex2009v12}
Cplex, I.I.: V12. 1: User’s manual for cplex. International Business Machines Corporation  (2009)

\bibitem{pmlr-v124-deshpande20a}
Deshpande, A., Kacham, P., Pratap, R.: Robust $k$-means++. In: Proceedings of the 36th Conference on Uncertainty in Artificial Intelligence (UAI) (2020)

\bibitem{UCI}
Dheeru, D., Taniskidou, E.K.: machine learning repository, 2017 (2017)

\bibitem{localout}
Gupta, S., Kumar, R., Lu, K., Moseley, B., Vassilvitskii, S.: Local search methods for k-means with outliers. Proc. VLDB Endow.  (2017)

\bibitem{han2023approx}
Han, L., Xu, D., Xu, Y., Yang, P.: Approximation algorithms for the individually fair k-center with outliers. J. of Global Optimization  (2022)

\bibitem{huang2024nearlinear}
Huang, J., Feng, Q., Huang, Z., Xu, J., Wang, J.: Near-linear time approximation algorithms for k-means with outliers. In: Forty-first International Conference on Machine Learning (2024)

\bibitem{huang2019coresetsclusteringfairnessconstraints}
Huang, L., Jiang, S.H.C., Vishnoi, N.K.: Coresets for clustering with fairness constraints (2019)

\bibitem{im2020fastnoiseremovalkmeans}
Im, S., Qaem, M.M., Moseley, B., Sun, X., Zhou, R.: Fast noise removal for k-means clustering. CoRR  (2020)

\bibitem{DBLP:journals/corr/abs-1908-09041}
Jung, C., Kannan, S., Lutz, N.: A center in your neighborhood: Fairness in facility location. CoRR  \textbf{abs/1908.09041} (2019)

\bibitem{krishnaswamy2018constantapproximationkmediankmeans}
Krishnaswamy, R., Li, S., Sandeep, S.: Constant approximation for $k$-median and $k$-means with outliers via iterative rounding (2018)

\bibitem{liu2008iforest}
Liu, F.T., Ting, K.M., Zhou, Z.H.: Isolation forest. In: 2008 Eighth IEEE International Conference on Data Mining. pp. 413--422 (2008). \doi{10.1109/ICDM.2008.17}

\bibitem{mahabadi2020ind}
Mahabadi, S., Vakilian, A.: Individual fairness for k-clustering. In: Proceedings of the 37th International Conference on Machine Learning (2020)

\bibitem{misc_bank_marketing_222}
Moro: {Bank Marketing}. UCI Machine Learning Repository (2012)

\bibitem{negahbani2021better}
Negahbani, M., Chakrabarty, D.: Better algorithms for individually fair \$k\$-clustering. In: Advances in Neural Information Processing Systems (2021)

\bibitem{vakilian2022improved}
Vakilian, A., Yalciner, M.: Improved approximation algorithms for individually fair clustering. In: Proceedings of The 25th International Conference on Artificial Intelligence and Statistics (2022)

\end{thebibliography}

% \printbibliography
% \bibliography{aaai25}
\endgroup

\end{document}